\documentclass[accepted]{uai2021} 

\usepackage[american]{babel}

\usepackage{natbib} 
\usepackage{mathtools} 
\usepackage{booktabs} 
\usepackage{tikz} 

\usepackage{graphicx}
\usepackage{times}
\usepackage{epsfig}
\usepackage{graphicx}
\usepackage{amsmath}
\usepackage{amssymb}
\usepackage{amsthm}
\usepackage{caption}
\usepackage{subcaption}
\usepackage{wrapfig}
\newtheorem{theorem}{Theorem}
\newtheorem{definition}{Definition}
\newtheorem{lemma}{Lemma}

\makeatletter
\@namedef{ver@everyshi.sty}{}
\makeatother
\usepackage{tikz, pgf}
\usetikzlibrary{calc}
\usetikzlibrary{arrows}
\usetikzlibrary{positioning}
\usetikzlibrary{shapes.geometric}
\usetikzlibrary{snakes}
\usetikzlibrary{decorations.pathreplacing,calc}
\usetikzlibrary{arrows.meta}
\usepackage{multirow}

\title{CInC Flow: Characterizable Invertible 3$\times$3 Convolution} 

\author[1]{Sandeep Nagar}
\author[2]{Marius Dufraisse}
\author[1]{Girish Varma}
\affil[1]{%
    Machine Learning Lab\\
    International Institute of Information Technology\\
    Hyderabad, India
}
\affil[2]{%
    Computer Science Dept.\\
    École Normale Supérieure (ENS)\\
    Paris-Saclay, France
}

\begin{document}
\maketitle

\begin{abstract}
Normalizing flows are an essential alternative to GANs for generative modelling, which can be optimized directly on the maximum likelihood of the dataset. They also allow computation of the exact latent vector corresponding to an image since they are composed of invertible transformations. However, the requirement of invertibility of the transformation prevents standard and expressive neural network models such as CNNs from being directly used. Emergent convolutions were proposed to construct an invertible 3$\times$3 CNN layer using a pair of masked CNN layers, making them inefficient. We study conditions such that 3$\times$3 CNNs are invertible, allowing them to construct expressive normalizing flows. We derive necessary and sufficient conditions on a padded CNN for it to be invertible. Our conditions for invertibility are simple, can easily be maintained during the training process. Since we require only a single CNN layer for every effective invertible CNN layer, our approach is more efficient than emerging convolutions. We also proposed a coupling method, Quad-coupling. We benchmark our approach and show similar performance results to emergent convolutions while improving the model's efficiency.
\end{abstract}
\section{Introduction}\label{sec:intro}
The availability of large datasets has resulted in improved machine learning solutions for more complex problems. However, supervised datasets are expensive to create. Hence unsupervised methods like generative models are increasingly worked on. Generative models that have been proposed can be broadly categorized under Likelihood-based methods and Generative Adversarial Methods. For example, an optimization algorithm could directly minimize the former's negative log-likelihood of the unsupervised examples. At the same time, in the latter, the loss function itself is modelled as a discriminator network that is trained alternatively. Hence likelihood-based methods directly optimize the probability of examples. In contrast, in GANs, the function being optimized is implicit and hard to reason about.


An essential type of Likelihood-based Generative models is normalizing flow-based models. Normalizing flow-based models transform a latent vector usually sampled from a continuous distribution like the Gaussian by a sequence of invertible functions to produce the sample. Hence even though the latent vector distribution is simple, the sample distribution could be highly complex, provided we are using an expressive set of invertible transformations. Also, the invertibility of the model implies that one can find the exact latent vector corresponding to an example from a dataset. All other approaches to generative modelling can compute the latent vector, for example, only approximately.

The ability of a normalizing flow based model to express complex real-world data distributions depends on the expressive power of the invertible transformations used. In supervised models in vision tasks, complex, multilayered CNNs with different window sizes are used. CNN's with larger window size helps in spatial mixing of information about the images, resulting in expressive features. Glow used invertible 1$\times$1 convolutions to build normalizing flow models \cite{glow}. For a 1$\times$1 convolution (if it is invertible), the inverse is also a 1$\times$1 convolution. We show that this approach does not generalize to larger window sizes. In particular, the inverse of an invertible 3$\times$3 convolution necessarily depends on all the feature vector dimensions, unlike CNNs, which only require local features.

Emerging convolutions proposed a way of inverting convolutions with large window sizes \cite{emerging}. The inverse is not a convolution and is computed by a linear equation system that can efficiently be solved using back substitution. However, for obtaining an invertible convolution, they required 2 CNN filters to be applied. Hence for every effective invertible convolution, they are required to do two convolutions back to back. We propose a simple approach using padding of obtaining invertible convolutions, which only uses a single convolutional filter. Furthermore, we are able to give a characterization (necessary and sufficient conditions) for the convolutions to be invertible. This allows us to optimize over the space of invertible convolutions during training directly. 


\paragraph{Main Contributions.}
\begin{itemize}
    \item We give necessary and sufficient conditions for a 3$\times$3 convolution to be invertible by making some modifications to the padding (see Section \ref{sec:inv-conv}).
    \item We also propose a more expressive coupling mechanism called quad coupling (see Section \ref{sec:quad-coupling}).
    \item We use our characterization and quad-coupling to train flow-based models that give samples of similar quality as previous works while improving upon the run-time compared to the other invertible 3$\times$3 convolutions proposed (see Section \ref{sec:results}).
\end{itemize}

\section{Related works}\label{sec:related_work}
\paragraph{Normalizing flows.}
A normalizing flow aims to model an unknown data distribution (\cite{glow,cubic_spline_flow,neural_spline_flow}), that is, to be able to sample from this distribution and estimate the likelihood of an element for this distribution.

To model the probability density of a random variable $x$, a normalizing flow apply an invertible change of variable $x=g_\theta(z)$ where $z$ is a random variable following a known distribution for instance $z\sim \mathcal{N}(0,I_d)$. Then we can get the probability of $x$ by applying the change of variable formula
\[  p_\theta\left(x\right)= p\left(f_\theta(x)\right)\left(\left|\frac{\partial f_\theta\left(x\right)}{\partial x^T}\right| \right) \]
where $f_\theta$ denotes the inverse of $g_\theta$ and $\left|\frac{\partial f_\theta\left(x\right)}{\partial x^T}\right|$ its Jacobian.

The parameters $\theta$ are learned by maximizing the actual likelihood of the dataset. At the same time, the model is designed so that the function $g_\theta$ can be inverted and have its jacobian computed in a reasonable amount of time.

\paragraph{Glow.}
RealNVP defines a normalizing flow composed of a succession of invertible steps (\cite{realNVP}). Each of these steps can be decomposed into layers steps. Improvements for some of these layers where proposed in later articles (\cite{glow}, \cite{emerging}).

\emph{Actnorm}: The actnorm layer performs an affine transformation similar to batch normalization. First, its weights are initialized so that the output of the actnorm layer has zero mean and unit variance. Then its parameters are learned without any constraint.

\emph{Permutation}: RealNVP proposed to use a fixed permutation to shuffle the channels as the coupling layer only acts on half of the channels. Later, \cite{glow} replaced this permutation with a 1$\times$1 convolution in Glow. These can easily be inverted by inverting the kernel. Finally, \cite{emerging} replaced this 1x1 convolution with the so-called emerging convolution. These have the same receptive convolution with a kernel of arbitrary size. However, they are computed by convolving with two successive convolutions whose kernel is masked to help the inversion operation.

\emph{Coupling layer}: The coupling layer is used to provide flexibility to the architecture. The Feistel scheme (\cite{feistelGeneralized}) inspires its design. They are used to build an invertible layer out of any given function $f$. Here $f$ is learn as a convolutional neural network. 
\[ y = [y_1, y_2], \quad y_1=x1, \quad y_2=(x_2 + f(x_1))*\exp(g(x_1)) \]
Where we get $x_1$ and $x_2$ by splitting the input $x$ along the channel axis.

\paragraph{Invertible Convolutional Networks.}
Complementary to normalizing flows, there has been some work done designing more flexible invertible networks. For example, \cite{revnet} proposed reversible residual networks (RevNet) to limit the memory overhead of backpropagation, while (\cite{JacobsenBZB19}) built modifications to allow an explicit form of the inverse, also studying the ill-conditioning of the local inverse. \cite{Flow++} proposed a flow-based model that is the non-autoregressive model for unconditional density estimation on standard image benchmarks

\emph{Invertible 1$\times$1 Convolution:}
\cite{glow} proposed the invertible 1$\times$1 convolution replacing fixed permutation (\cite{realNVP}) that reverses the ordering of the channels. \cite{emerging} proposed normalizing flow method to do the inversion of 1$\times$1 convolution with doing some padding on the kernel and two distinct auto-regressive convolutions, which also provide a stable and flexible parameterization for invertible 1$\times$1 convolutions.  

\emph{Invertible n$\times$n Convolution:} Reformulating n$\times$n convolution using invertible shift function proposed by \cite{glow_nxn} to decrease the number of parameters and remove the additional computational cost while keeping the range of the receptive fields. In our proposed method, there is no need for the reformulation of standard convolutions. \cite{emerging} proposed two different methods to produce the invertible convolutions : (1) Emerging Convolution and (2) Invertible Periodic Convolutions. Emerging requires two autoregressive convolutions to do a standard convolution, but our method requires only one convolution as compare to the method proposed by \cite{emerging} and increase the flexibility of the invertible n$\times$n convolution. 
\section{Our approach}\label{sec:our_approach}
We propose a novel approach for constructing invertible 3$\times$3 convolutions and coupling layers for normalizing flows. We propose two modifications to the existing layers used in previous normalizing flow models:
\begin{itemize}
	\item convolution layer: instead of using 1$\times$1 convolutions or emerging convolutions, we propose to use standard convolutions with a kernel of any size with a specific padding. 
	\item coupling layer: we propose to use a modified version of the coupling layer designed to have a bigger receptive field.
\end{itemize}
We also show how invertibility can be used to manipulate images semantically.

\subsection{Invertible $3\times3$ Convolution}
\label{sec:inv-conv}
We give necessary and sufficient conditions for an arbitrary convolution with some simple modifications on the padding to be invertible. Moreover, the inverse can also be computed by an efficient back substitution algorithm.

\begin{definition}[Convolution]
The convolution of an input $X$ with shape $H\times W \times C$ with a kernel $K$  with shape $k\times k\times C\times C$ is $Y=X*K$ of shape $(H-k+1)\times (W-k+1)\times C$ which is equal to 
\begin{equation}
   Y_{i,j,c_o} = \sum_{l,h < k}\sum_{c_i=1}^{C}I_{i+l,j+k,c_i}K_{l,k,c_i,c_o}
\end{equation}.
\end{definition}

In this setting, the output $Y$ has a smaller size than the input to prevent this input is padded before applying the convolution.
\begin{definition}[Padding]
Given an image $I$ with shape $H\times W \times C$, the $(t,b,l,r)$ padding of $I$ is the image $\hat{I}$ of shape $(H + t + b)\times (W+l+r) \times C$ defined as
    \begin{equation}
        \hat{I}_{i,j,c} = \begin{cases}
        I_{i-t,j-l,c}\ &\text{if } i-t < H\text{ and } j-l < W\\
        0&\text{otherwise}
        \end{cases}
    \end{equation}
\end{definition}
\begin{figure*}[ht]
	\centering
		\includegraphics[width=0.89\linewidth]{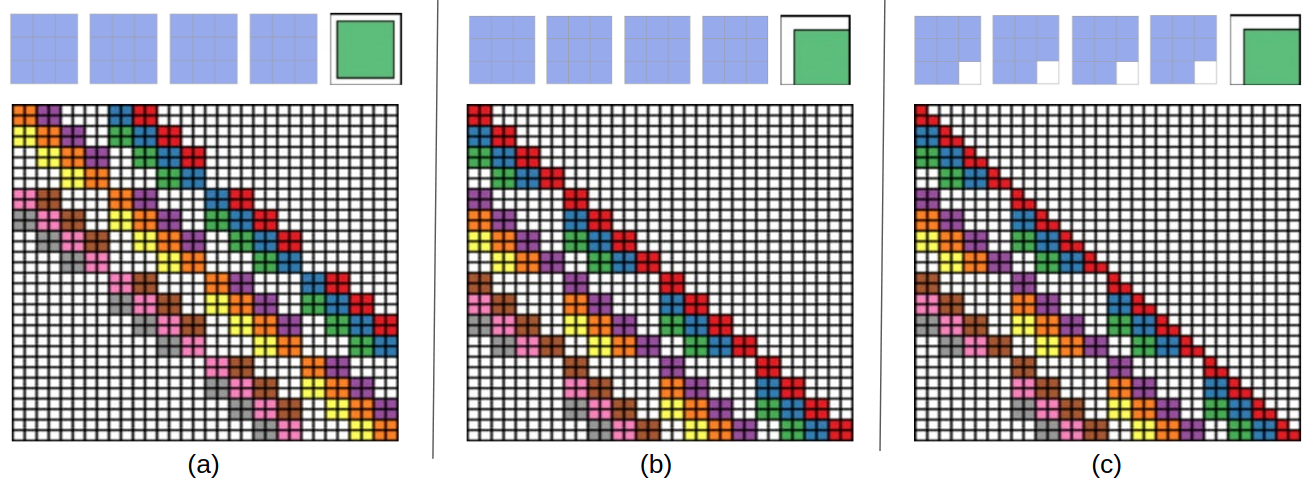}	
	\caption{(a).Top: first four is kernel matrix and fifth is input matrix with the \emph{standard} padding that give the bottom convolution matrix,  Bottom: the convolution matrix corresponding to a convolution with kernel of size 3 applied to an input of size $4\times4$ padded on both sides and with 2 channels. Zero coefficients are drawn in white, other coefficient are drawn using the same color if they are applied to the same spatial location albeit on different channels. (b) Top: an \emph{alternative} padding scheme that results in a block triangular matrix $M$, Bottom: The matrix corresponding to a convolution with kernel of size 3 applied to an input of size $4\times4$ padded only on one side and with 2 channels. (c) Top: an \emph{masked alternative} padding scheme that results in a  triangular matrix $M$, Bottom: the matrix corresponding to a convolution with kernel of size 3 applied to an input of size $4\times4$ padded only on one side and with 2 channel. One of the weight of the kernel is masked. Note that the equivalent matrix $M$ is \emph{triangular}.}
	\label{fig:matrix}
\end{figure*}
As zero padding does not add any bias to the input, the convolution between a padded input $\hat{I}$ and a kernel $K$ is still a linear map between the input and the output. As such, it can be described as matrix multiplication.

An image $I$ of shape $(H,W,C)$ can be seen as an vector $\vec{I}$ of $\mathbb{R}^{H\times W\times C}$. In the rest of this paper we will always use the following basis 
$I_{i,j,c} = \vec{I}_{c + Cj + CHi}$.  For any index $i\leq HWC$, let $(i_y,i_x,i_c)$ denote the indexes that satisfy $\vec{I}_i=I_{i_y,i_x,i_c}$. Note that $i < j$ iff $(i_y,i_x,i_c) \prec (j_y,j_x,j_c)$ where $\prec$ denotes the lexicographical order over $\mathbb{R}^3$. If $C=1$ this means that the pixel $(j_y,j_x)$ is on the right or below the pixel  $(i_y,i_x)$.

\begin{definition}[Matrix of a convolution.]
Let $K$ be a kernel of shape $k\times k \times C \times C$. The matrix of a convolution of kernel $K$ with input $X$ of size $H\times W\times C$ with padding $(t,b,l,r)$ is a matrix describing the linear map $X\mapsto \hat{X}*K$.
\end{definition}

\paragraph{Characterization of invertible convolutions:}
We consider convolution with top and left padding only. For such convolutions, we give necessary and sufficient conditions for it to be invertible. 
Let $K$ be the kernel of the convolution with shape $3\times 3 \times N \times N$ where $3\times3$ is the window size, and $N$ is the number of channels. Note that number of input channels should be equal to the number of output channels for it to be invertible.
\begin{lemma}\label{lem:lower-triang}
\textbf{$M$} is a lower triangular matrix with all diagonal entries $=K_{3,3}$
\end{lemma}
Where the matrix $M$ is which produces the equivalent result when multiplied with a vectorized input ($\hat{x}$).

\begin{proof}
Consider any entry in the upper right half of $M$. That is $(i,j)$ such that $i < j$ according to the ordering given in the definition of $M$. $M_{i,j}$ is nothing but the scalar weight that needs to be multiplied to the $j$th pixel of input when computing $i$th pixel of the output. The linear equation relating these two variable is as follows:
$$ y_i = \sum_{l=0}^3\sum_{k=0}^3K_{3-l,3-k}x_{i_x-l,i_y-k} $$
From this equations follows that if $j_x>i_x$ or $j_y>i_y$ then the $i$th pixel of the output does not depend on the $j$th pixel of the input and thus $M_{i,j}=0$. This also justifies that all diagonal coefficients of $M$ are equal to $K_{3,3}$
\end{proof}

We first describe our conditions for the case when $N=1$. We prove the following theorem.
\begin{theorem}[Characterization for $N=1$] \label{thm:inv-conv-1}
$$M \text{ is invertible iff } K_{3,3} \neq 0.$$
\end{theorem}
\begin{proof}
The proof of the theorem uses Lemma \ref{lem:lower-triang}.  Since $M$ is lower triangular, the determinant is nothing but the product of diagonal entries, which is $=K_{3,3}^{h*w}$ where $h,w$ is the dimensions of the input/output image. 
\end{proof}



At its core, the convolution layer is a linear operation. However, we have no guarantees regarding it as invertibility. The result $z$ of the convolution of input $x$ with kernel $k$ can be expressed as the product as $x$ with a matrix $M$. When zero-padding is used around the input so that $x$ and $z$ have the same shape, the matrix $M$ is not easily invertible because the determinant of $M$ can be zero (see matrix $M$ in Figure \ref{fig:matrix}(a)).

However, when padding only on two sides (left and top), the corresponding $M$ is blocked triangular (see Figure \ref{fig:matrix}(b)). To further ensure invertibility and speed up the inversion process, we also mask part of the kernel so that the matrix corresponding to the convolution is triangular, see in Figure \ref{fig:matrix}(c). In this configuration, the jacobian of the convolution can also be easily computed.


\subsection{Quad-coupling}
\label{sec:quad-coupling}
The coupling layer is used to have some flexibility as its functions can be of any form. However, it only combines the effects of half channels. 
To overcome this issue we designed a new coupling layer inspired from generalized Feistel (\cite{feistelGeneralized}) schemes. Instead of dividing the input $x$ into two blocks we divide it into four $x = \left[x_1,x_2,x_3,x_4\right]$ along the feature axis. Then we keep $x_1$ unchanged  and use it to modify the other blocks in an autoregressive  manner (see Figure \ref{fig:coupling}):
\begin{align}
y_1 =& x_1\\
y_2 =& (x_2 + f_1(x_1)) * \exp(g_1(x_1))\\
y_3 =& (x_3 + f_2(x_1,x_2)) * \exp(g_2(x_1,x_2))\\
y_4 =& (x_4 + f_3(x_1,x_2,x_3)) * \exp(g_3(x_1,x_2,x_3))
\end{align}
where $(f_i)_{i\leq 3}$ and $(g_i)_{i\leq 3}$ are learned. The output of the layer is obtained by concatenating the $(y_i)_{i\leq 4}$.
\begin{figure}[ht!]
\centering
\includegraphics[width=0.65\linewidth]{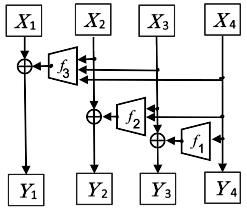}
   \caption{The quad-coupling layer, each input block $X_i$ has the same spatial dimension as the input $X$ but only one quarter of the channels. Each of the function $f_1$, $f_2$ and $f_3$ is a 3 layer convolutional network. $\bigoplus$ symbolizes a component-wise addition. The multiplicative actions are not represented here.}
	\label{fig:coupling}

\end{figure}
\begin{figure}[ht!]
\centering
\includegraphics[width=0.85\linewidth]{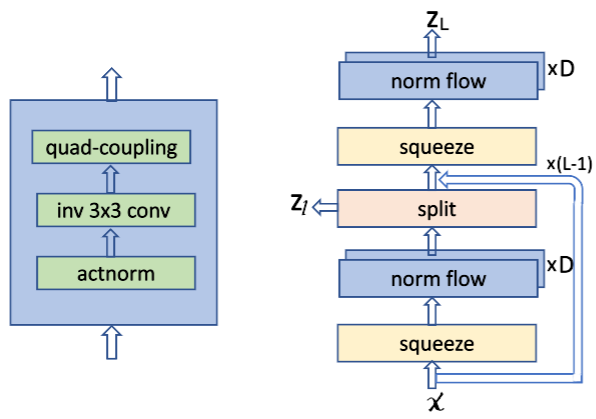}
   \caption{Overview of the model architecture. Left, the flow modules we propose: containing inv $3\times3$ convolution. The diagram on the right shows the entire model architecture, where the flow module is now grouped. The squeeze module reorders pixels by reducing the spatial dimensions by a half, and increasing the channel depth by four. A hierarchical prior is placed on part of the intermediate representation using the split module as in (\cite{glow}). $x$ and $z$ denote input and output. The model has L levels, and D flow modules per level.}
	\label{fig:norm_flow}
\end{figure}

\section{Experimental results}\label{sec:results}
The architecture is based on \cite{emerging}. We modified the emerging convolution layer to use our standard convolution. We also introduced the quad-coupling layer in place of the affine coupling layer. Finally, we evaluate the model on a variety of models and provide images sampled from the model. For detailed overview of the architecture see Figure \ref{fig:norm_flow}.

\paragraph{Training setting:} To train the model on Cifar10, we used the 3 level (L) and depth (D) of 32 and lr $0.001$ for the 500 epochs. To train on ImageNet32, $L = 3$, $D = 48$, lr $0.001$ for the 600 epochs and for ImageNet64, $L = 4$, $D = 48$, lr $0.001$ for the 1000 epochs. See Figure \ref{fig:norm_flow} for the model architecture.
\paragraph{Quantitative results:}
The Comparison of the performance of  $3\times3$ invertible convolution with the emerging convolution (\cite{emerging}) for the cifar10 dataset in Table \ref{tab:cifar}. The performance of our layers was tested on CIFAR10 (\cite{cifar10}), ImageNet (\cite{imagenet}) as well as on the galaxy dataset (\cite{galaxy}) see Table \ref{tab:performance}. We also tested our architecture on networks with a smaller depth ($D=4$ or $D=8$) see Table \ref{tab:smallnet} which could be used when computational resources are limited as their sampling time is much lower. In this case, using standard convolution and quad-coupling offers a more considerable performance improvement than with bigger models (see Table \ref{tab:smallnet}). 

\begin{table}[ht!]
    \centering
    \begin{tabular}{|c|c|c|}
    		\hline
    		 &  Emerging 3$\times$3 Inv. conv & Our 3$\times$3 Inv. conv \\
    		\hline
    		Affine & 3.3851 & 3.4209 \\
    		\hline
    		Quad & \textbf{3.3612} & \textbf{3.3879} \\
    		\hline
    	\end{tabular} 
    	\caption{Comparison of the performance (in bits per dimension) achieved on the Cifar10 dataset with different coupling architectures.}
    	\label{tab:cifar}
\end{table}    

\begin{table}[ht!]
    \centering
     \begin{tabular}{|c|c|c|c|c|}
    		\hline
    		& Glow & Emerging & $3\times 3$ & Quad\\
    		\hline 
    		Cifar10 & 3.35 & 3.34 & \textbf{3.3498} & \textbf{3.3471} \\ 
    		\hline
    		ImageNet32 & 4.09 & 4.09 & \textbf{4.0140} & 4.0377 \\ 
    		\hline
    		ImageNet64 & 3.81 & 3.81 & 3.8946 & 3.8514 \\ 
    		\hline
    		Galaxy & --- & 2.2722 & 2.2739 & \textbf{2.2591} \\ 
    		\hline 
    	\end{tabular} 
    	\caption{Performance achieved on the Cifar10 and Imagenet datasets after a limited number of epochs (500 for Cifar10, 600 for ImageNet32,ImageNet64 and 1000 for Galaxy). Emerging results were obtained by using the code provided in \cite{emerging}, $3\times 3$ is replacing the emerging convolutions by our $3\times 3$ invertible convolutions and quad uses quad-coupling on top of this.}
    	\label{tab:performance}
\end{table}

\begin{figure}[ht!]
\centering
		\includegraphics[width=0.8\linewidth]{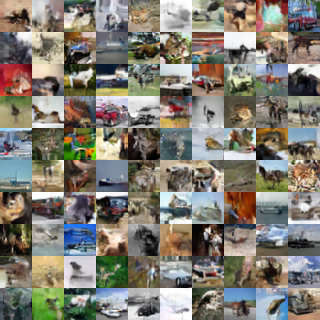}

	\caption{Sample images generated after training on the cifar dataset.}
	\label{fig:sample_cifar}
\end{figure}


\begin{table*}
\centering
	\begin{tabular}{|c|c|c|c|c|c|}
		\hline  
		\multirow{2}{*}{Dataset} & \multicolumn{2}{c|}{Emerging} & \multicolumn{2}{c|}{Ours} & \multirow{2}{*}{Depth}  \\
		       & Performance & Sampling time & Performance & Sampling time &   \\ 
		\hline
		Cifar10 &  3.52 & \multirow{2}{*}{2.45} & \textbf{3.49} & \multirow{2}{*}{\textbf{1.31}} & \multirow{2}{*}{4} \\ 
		\cline{1-2} \cline{4-4}
		{Imagenet32} & 4.30 &  & \textbf{4.25} &  & \\ 
		\hline 
		Cifar10 &  3.47 & \multirow{2}{*}{4.94} & \textbf{3.46} & \multirow{2}{*}{\textbf{2.76}} & \multirow{2}{*}{8} \\ 
		\cline{1-2} \cline{4-4}
		Imagenet32 & 4.20 &   & \textbf{4.18} &  & \\ 
		\hline 
	\end{tabular}
	\vspace{-1em}
	\caption{Performance with smaller networks, when computational resources are limited. The performance is expressed in bits per dimension and the sampling time is the time in seconds needed to sample 100 images. All networks were trained for 600 epochs.}
	\label{tab:smallnet}
\end{table*}

\begin{figure*}[ht]
\centering
	\includegraphics[width=0.89\textwidth]{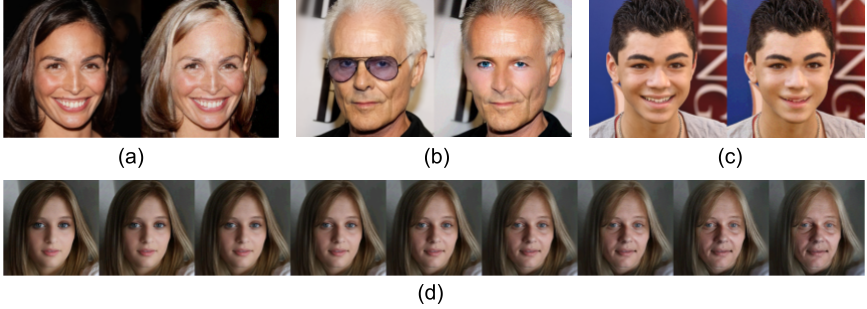}
	\vspace{-1em}
	\caption{From left to right : the result obtained when using the network to change hair colour (a), remove glasses (b), and visage shape (c). For every example, the original image is shown on the left. Fig.(d) here, we can see the result of gradually modifying the age parameter. The original image is the fourth from the left (middle one).}
	
	\label{fig:midif}
	
\end{figure*}

\paragraph{Sampling Times:}
We compared our method's sampling time (Table \ref{tab:samplingtimes}) against Glow (\cite{glow}) and Emerging Convolutions (\cite{emerging}). Our convolution still requires solving a sequential problem to be inverted and, as such, need to be inverted on CPU, unlike Glow that can be inverted while performing all the computation on  GPU. This explains the gap between the sampling time of our model compared to Glow. However, it is still roughly two times faster than emerging convolutions; this comes from the need to solve two inversion problems to invert one emerging convolution layer. The quad-coupling layer does not affect sampling time too much.
\begin{table}[ht!]
\centering
	\begin{tabular}{|c|c|c|c|c|}
		\hline 
		& Glow & Emerging & $3\times 3$  & Quad \\ 
		\hline 
		Cifar10 & 0.58 & 18.4 & 9.3 & 10.8 \\ 
		\hline 
		Imagenet32 & 0.86 & 27.6 & 14.015 & 16.1 \\ 
		\hline
		Imagenet64 & 0.50 & 160.72 & 82.04 & 84.06 \\ 
		\hline
	\end{tabular}
	\vspace{-0.5em}
	\caption{Time in seconds to sample 100 images. Results were obtained with Glow running on GPU and the other methods running on one CPU core.}
	\label{tab:samplingtimes}
\end{table}

\paragraph{Interpretability results:}
To show the interpretability of our invertible network, we used the Celeba dataset (\cite{celeba}) which provides images of faces and attributes corresponding to these faces. In Figure \ref{fig:sample_cifar} are the randomly generated fake sample images for the cifar10 dataset. The covariance matrix between the attributes of images in the dataset and their latent representation indicates how to modify the latent representation of an image to add or remove features. Examples of such modifications can be seen in Figures \ref{fig:midif}(a, b, c) and \ref{fig:midif}(d).

\section{Conclusion}
In this paper, we propose a new method for Invertible n$\times$n Convolution. Coupling layers solve two problems for normalizing flows: they have a tractable Jacobian determinant and can be inverted in a single pass. We propose a new type of coupling method, Quad-coupling. Our method shows consistent improvement over the Emerging convolutions method, and we only need a single CNN layer for every effective invertible convolution. This paper shows that we can invert a convolution with only one effective convolution, and additionally, the inference time and sampling time improved notably. We show the inversion of 3$\times$3 convolution and the generalization of the inversion for the n$\times$n kernel.  Furthermore, we demonstrate improved quantitative performance in terms of log-likelihood on standard image modelling benchmarks.




\bibliography{uai2021-template}

\appendix
\end{document}